\documentclass[letterpaper]{article} 
\usepackage{aaai25}  
\usepackage{times}  
\usepackage{helvet}  
\usepackage{courier}  
\usepackage[hyphens]{url}  
\usepackage{graphicx} 
\usepackage{natbib}  
\usepackage{caption} 
\usepackage{algorithm}
\usepackage{algorithmic}

\usepackage{newfloat}
\usepackage{listings}
\DeclareCaptionStyle{ruled}{labelfont=normalfont,labelsep=colon,strut=off} 
\floatstyle{ruled}
\newfloat{listing}{tb}{lst}{}
\floatname{listing}{Listing}
\usepackage{booktabs} 
\usepackage[english]{babel}
\usepackage{amsmath,amssymb,amsthm}
\usepackage{pgfplots}
\usepackage{pgfplots}
\pgfplotsset{compat=1.18}  
\usepgfplotslibrary{fillbetween}
\usepgfplotslibrary{groupplots}
\usepackage{pgfplotstable}
\usepackage{siunitx} 
\usepackage{array}
\usepackage{multirow}
\usepackage{arydshln}
\usepackage{tikz}
\usetikzlibrary{decorations.text}
\usetikzlibrary{patterns}
\usetikzlibrary{intersections}
\usepackage{subcaption}
\usepackage{colortbl}  

\usepackage{algorithm}
\usepackage{algorithmic} 

\theoremstyle{definition}
\newtheorem{definition}{Definition}

\theoremstyle{remark}
\newtheorem*{remark}{Remark}

\theoremstyle{plain}

\newtheorem{proposition}{Proposition}

\newcommand\independent{\protect\mathpalette{\protect\independenT}{\perp}}
\def\independenT#1#2{\mathrel{\rlap{$#1#2$}\mkern2mu{#1#2}}}

\title{Balancing Profit and Fairness in Risk-Based Pricing Markets}
\author{
    Jesse Thibodeau\textsuperscript{\rm 1},
    Hadi Nekoei\textsuperscript{\rm 1,2},
    Afaf Taïk\textsuperscript{\rm 1,2},
    Janarthanan Rajendran\textsuperscript{\rm 3},
    Golnoosh Farnadi\textsuperscript{\rm 1,4}
}
\affiliations{
    \textsuperscript{\rm 1}Mila - Qu\'{e}bec AI Institute\\
    \textsuperscript{\rm 2}Universit\'{e} de Montr\'{e}al\\
    \textsuperscript{\rm 3}Dalhousie University\\
    \textsuperscript{\rm 4}McGill University\\

    \{\texttt{jesse.thibodeau, nekoeihe, afaf.taik, farnadig\}@mila.quebec}\\
    \texttt{janarthanan.rajendran@dal.ca}
}

\begin{document}

\maketitle

\begin{abstract}
Dynamic, risk-based pricing can systematically exclude vulnerable consumer groups from essential resources such as health insurance and consumer credit.  We show that a regulator can realign private incentives with social objectives through a learned, interpretable tax schedule.  First, we provide a formal proposition that bounding each firm’s \emph{local} demographic gap implicitly bounds the \emph{global} opt-out disparity, motivating firm-level penalties.  Building on this insight we introduce \texttt{MarketSim}—an open-source, scalable simulator of heterogeneous consumers and profit-maximizing firms—and train a reinforcement learning (RL) social planner (SP) that selects a bracketed fairness-tax while remaining close to a simple linear prior via an \(\ell_1\) regularizer.  The learned policy is thus both transparent and easily interpretable. In two empirically calibrated markets, i.e., U.S. health-insurance and consumer-credit, our planner simultaneously raises demand-fairness by up to 16\% relative to unregulated Free Market while outperforming a fixed linear schedule in terms of social welfare without explicit coordination. These results illustrate how AI-assisted regulation can convert a competitive social dilemma into a win–win equilibrium, providing a principled and practical framework for fairness-aware market oversight.
\end{abstract}

\section{Introduction}
\label{sec:intro}

Firms equipped with modern computational power and extensive consumer data logs may reap financial gains by adopting dynamic (or personalized) pricing, which tailors prices to potential customers or customer segments based on their estimated willingness-to-pay. This approach enables firms to extract the greatest economic value from consumer data. From an efficiency perspective, dynamic pricing has been shown to boost firm profitability and accelerate sales speeds~\citep{schlosser2018dynamic, wang2023algorithms}. However, its welfare implications are less consistent. In some markets, including insurance and lending, dynamic pricing can yield undesirable distributional outcomes~\citep{zhu2023neoliberalization, betancourt2022dynamic}. For instance, while health insurers often rely on dynamic pricing, recent census data indicate that members of the Hispanic population in the U.S. are, on average, roughly twice as unlikely to have healthcare coverage as members of the Afro-descendent population, who in turn are twice as unlikely as members of the Caucasian and Asian populations~\citep{quickstats, keisler2024health}. Similarly, data reveal a negative correlation between likelihood of coverage and income, and since income and ethnicity are themselves correlated, there are justifiable concerns that healthcare coverage may be systemically biased.

\begin{figure}[ht!]
\centering
\begin{tikzpicture}
\begin{groupplot}[
    group style={
        group size=2 by 1,
        horizontal sep=0.2cm,
    },
    ybar,
    ymin=0, ymax=39,
    width=0.3\textwidth,
    height=4.5cm,
    enlarge x limits=0.2,
    tick style={draw=none},      
    yticklabels={},               
    every node near coord/.append style={font=\scriptsize,anchor=south},
    ticklabel style={font=\scriptsize},
]

\nextgroupplot[
    title={\scriptsize \textbf{Race \& Hispanic Origin}},
    ylabel={\scriptsize \% Uninsured},
    symbolic x coords={White,Black,Asian,Hispanic (any race)},
    xtick=data,
    xticklabel style={rotate=25,anchor=east},
    nodes near coords,
]
\addplot+[fill=blue!60] coordinates {
    (White,7.0)
    (Black,11.1)
    (Asian,6.8)
    (Hispanic (any race),23.6)
};

\nextgroupplot[
    title={\scriptsize \textbf{Income}},
    symbolic x coords={Lowest quintile,Second quintile,Third quintile,Fourth quintile,Highest quintile},
    xtick=data,
    xticklabel style={rotate=25,anchor=east},
    nodes near coords,
]
\addplot+[fill=blue!60] coordinates {
    (Lowest quintile,21.2)
    (Second quintile,18.8)
    (Third quintile,13.1)
    (Fourth quintile,7.4)
    (Highest quintile,3.8)
};

\end{groupplot}
\end{tikzpicture}
\caption{Percentage of working-age adults without health insurance in 2023, by race and income.}
\label{fig:insurance_sidebyside}
\end{figure}
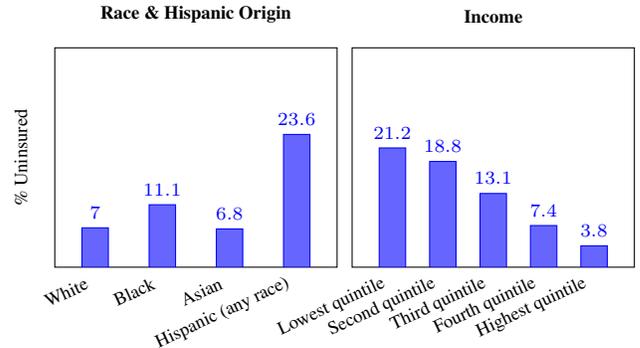

Beyond this, substantial price discrepancies between consumers may also give rise to perceptions of unfairness~\citep{lee2011perceived}, potentially discouraging market participation and perpetuating existing disparities. Moreover, in markets where personal assets can be leveraged to negotiate more favourable terms, goods may become \emph{relatively} more affordable for higher-income consumers. In such scenarios, scarce goods tend to be allocated to privileged groups, leaving fewer units—or lower-quality alternatives—to those with fewer resources. This pattern is well documented in the lending market, where it amplifies wealth gaps (for example, through restricted access to home equity) and drives gentrification. Similar issues arise in sectors such as education services and public transportation, where buyer distributions should ideally reflect those of the underlying population.
\\
\newline
Dynamic pricing typically enables firms to set prices for different consumer segments in order to maximize the expected profits derived from each. Thus, a firm adopting dynamic pricing rarely accounts for its resulting buyer distribution, which can be considered unfair if it diverges too sharply from that of the broader population. In this study, we explore the challenge of dynamic pricing, and specifically that of~\emph{regulating its use under demand fairness criteria} in markets where buyer distributions should mirror those of the underlying population. Motivating this, we first demonstrate how profit-maximizing price allocations fail to satisfy demand fairness under purely competitive or collusive dynamics. Broadly, we recognize that it is unrealistic to expect a profit-maximizing firm to voluntarily consider fairness notions in its pricing strategy. Therefore, we consider how a benevolent social planner (SP) might use policy tools such as taxation to encourage market participation among underrepresented consumer groups and penalize unfair firm behavior. To achieve this, we use reinforcement learning to train an SP capable of devising financial incentives that promote demand fairness—namely, by shrinking market opt-out disparities between population subgroups. To conduct our experiments, we introduce~\texttt{MarketSim}, a simple yet robust and scalable simulation framework wherein arbitrarily many firms engage in price Free Market to capture a market of arbitrarily many consumer profiles with heterogeneous utility. \textbf{Our findings indicate that social welfare can be enhanced by taxing firms in ways that incentivize fairer market-specific conduct}. Further, by endowing the learning agent with a domain-specific prior, we observe that interpretability can be maintained at the policy level, confirming the effectiveness of our approach at solving specific markets while retaining certain desirable properties such as tax monotonicity. Contributions of our work include:
\begin{itemize}
    \item A formal demonstration of how local, firm‑level incentives can satisfy a global fairness criterion and thereby raise social welfare in a setting of multiple self‑interested stakeholders with incomplete information.
    \item The introduction of \texttt{MarketSim}, a robust and easy-to-use open-source simulator for experimenting with various market dynamics and regulatory policies imposed on arbitrarily many heterogeneous firms and consumers, and evaluating their welfare implications.\footnote{Code will be made publicly available upon acceptance of this work.}
    \item An application of reinforcement learning to generate optimal regulatory policies in two different instances of \texttt{MarketSim} (replicating the markets for insurance and consumer credit), showing that we can incentivize welfare-improving firm behaviour in each market, while retaining policy interpretability.
\end{itemize}

\section{Related Work}
\label{sec:related_work}
The interdisciplinary nature of this work requires a review of topics from economics, specifically in the subfields of welfare economics and consumer choice theory, as well as a broad overview of applications of artificial intelligence (AI) to welfare economics.

\subsection{Economics Foundations}
In this work, we explore consumer choice and pricing dynamics in competitive markets with heterogeneous agents. On the demand side, consumers exhibit varying sensitivities to price fluctuations, while on the supply side, firms face heterogeneous marginal costs that proxy for technological and scale advantages. Our demand system borrows the random-utility framework of \citet{berry1993automobile}, yet departs from it in two key respects: first, we replace perfectly rational choice with a stochastic rule that captures bounded rationality and behavioural noise; second, we model multiple competing firms rather than a representative producer, thereby enabling firm-level strategic interaction. These extensions link our analysis to early work on preference heterogeneity by \citet{becker1962irrational} and its discrete-choice formalisation by \citet{mcfadden1972conditional}, whose insights remain central to modern consumer-choice theory \citep{ben2002hybrid}.  Finally, following the process-based welfare view of \citet{fleurbaey2008fairness}, we evaluate outcomes not only by efficiency but also by the fairness of the mechanisms that generate them, an angle largely absent from the original random-utility literature.

\subsection{Fairness in Dynamic Pricing}
The welfare-theoretic study of price design has migrated from economics to operations research \citep{gallego2019revenue} and computer science \citep{das2022individual}, giving rise to a rich taxonomy of fairness definitions.~\citet{cohen2022price} prove that price, demand, consumer-surplus, and no-purchase fairness cannot be simultaneously satisfied in dynamic settings; we therefore adopt \emph{demand fairness} \citep{cohen2022price,kallus2021fairness}, which directly measures disparate impact on group participation and is well motivated in education, consumer credit and healthcare domains.  Alternative notions such as proportional fairness \citep{bertsimas2011price} highlight welfare trade-offs but do not readily extend to sequential, multi-firm games.  RL approaches such as \citet{maestre2019reinforcement} impose fairness via Jain’s index under monopoly; in contrast, our regulator shapes \emph{competitive} firms’ incentives so that they voluntarily choose fairer prices, thereby filling the gap between single-seller RL treatments and static constrained-optimisation models.

\subsection{AI for Economic Policy Generation}
The closest application to our work combining economic simulations and sequential modelling is the AI Economist \citep{zheng2020ai}, where agents interact in a simulated gather-build-trade society, while a social planner aims to learn an income taxation strategy that improves social welfare, defined as the product of equality and economic productivity. While their work is effective at showcasing emergent behaviours among simulated consumer-workers under incumbent tax regimes, we introduce a new layer to our exploration which instead focuses on how a dynamic regulator can impact societal outcomes by aligning the objectives of self-interested firms with its own. Further, our work focuses on markets involving dynamic pricing, where firms assign prices based on consumer group membership. This added complexity allows for a deeper analysis of firm responses to incumbent policy frameworks. In addition, we refer to the safe RL literature to impose domain-specific policy constraints in order to maintain a degree of interpretability, which is critical for the widespread adoption of AI-generated public policy. In fact, safe RL methods routinely incorporate domain constraints to ensure an agent’s policy remains feasible or respects regulatory standards \citep{garcia2015comprehensive}. These constraints can be implicit (e.g., penalizing the agent for violating constraints)~\citep{achiam2017constrained}. By tethering our AI regulator's reward function to a baseline monotonic schedule, we encourage the learned regulatory policy to align with essential policy norms, namely, that higher fairness should generally not be penalized by higher tax rates, while preserving the safety and interpretability required in real-world economic regulation.

\section{Preliminaries}\label{sec:preliminaries}
We provide an overview of notation and definitions referred to throughout the remainder of this work. Among these, we refer to local and global notions of fairness within this context. Further, we motivate our policy mechanism design by demonstrating how local fairness incentives have global fairness implications. Going forward, let \( A \) be a random variable representing a consumer profile, taking values in the set \( \mathcal{A} = \{1, \dots, m\} \), and let $F$ be a random variable denoting the firm selected by a given consumer, taking values in \( \mathcal{F} = \{0, \dots, n\} \), with \(F=0\) referring to opting out of the market.

\subsection{Definitions and Fairness Metrics}\label{subsec:fair_metrics}
\begin{definition}[\(\epsilon\)-\textbf{Local Fairness}]\label{def:local-fairness}

For any firm \( j \in \mathcal{F} \), and for any consumer profile pair \(i,k \in A\), we say that firm \(j\) is \(\epsilon\)\emph{-locally fair} if
\[
\max_{i,k}\,\bigl|\Pr(F=j\;|\;A=i)-\Pr(F=j\;|\;A=k)\bigr|\;\le\;\epsilon.
\]
\end{definition}

\begin{remark}
When \(\epsilon=0\), it means that every consumer group's consumption choice is conditionally independent from their group membership.
\end{remark}
\noindent We quantify market-wide fairness via the opt-out disparity between consumer groups, which we call~\emph{global fairness} in order to relate it to its~\emph{local} counterpart. This definition draws on~\emph{demand fairness}, proposed by~\citet{cohen2021dynamic}.
\begin{definition}[$\epsilon$–\textbf{Global Fairness}]\label{def:global}
The entire market is \emph{$\epsilon$–globally fair} if the \emph{opt‑out
rate} is (approximately) profile–independent:
\[
\max_{i,k}\Bigl|
    \Pr(F=0\;|\;A=i)-\Pr(F=0\;|\;A=k)
\Bigr|\;\le\;\epsilon.
\]
\end{definition}

\begin{remark}
When $\epsilon=0$, every profile opts out with exactly the same probability,
i.e.\ $Pr(F=0)$ is constant in $A$. This result is akin to demographic parity~\cite{dwork2012fairness}, where profile \(i \independent\) firm \(j\).
\end{remark}

\subsection{Fairness alignment}\label{subsec:alignment}

Local and global fairness capture different levels of discrimination, though both measurements have shortcomings if considered on their own. On one hand, perfect local fairness fails to capture consumer counts, that is, a firm's consumers can mirror the population while including very few in total. Similarly, global fairness is satisfied under no market participation, i.e. the case where everyone opts out. Thus, any policy objective involving global fairness should also include some notion of economic productivity. Further, global fairness does not imply local fairness and is not directly addressable. Figure~\ref{fig:misalign} illustrates a market in which \emph{global} fairness
holds—both consumer profiles opt out $\sim$ 20\,\% of the time—yet the two firms serve very
different mixes of profiles ($F_1$ serving $A_1$ and $F_2$ serving $A_2$), violating \emph{local} fairness.
\begin{figure}[!h]
  \centering
  \begin{tikzpicture}[scale=0.8]
    \draw[thick] (0,1.4) rectangle (5,1.9);
    \draw[thick] (1,1.4) -- (1,1.9);
    \draw[thick] (4,1.4) -- (4,1.9);
    \fill[pattern=north west lines] (4,1.4) rectangle (5,1.9);
    \node at (0.5,1.65) {\scriptsize $F_2$};
    \node at (2.5,1.65) {\scriptsize $F_1$};
    \node at (5.3,1.65) {\scriptsize $A_1$};

    \draw[thick] (0,0.3) rectangle (5,0.8);
    \draw[thick] (3,0.3) -- (3,0.8);
    \draw[thick] (4,0.3) -- (4,0.8);
    \fill[pattern=north west lines] (4,0.3) rectangle (5,0.8);
    \node at (1.5,0.55) {\scriptsize $F_2$};
    \node at (3.5,0.55) {\scriptsize $F_1$};
    \node at (5.3,0.55) {\scriptsize $A_2$};
  \end{tikzpicture}
  \vspace*{.5em}
  \caption{Perfect global fairness does \emph{not} imply local fairness.}
  \label{fig:misalign}
\end{figure}
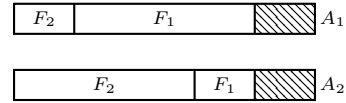
\noindent How then can an external regulator, unable to address global fairness explicitly, make global fairness improvements by deploying firm-level incentives? This challenge constitutes a mechanism design problem, which we formalize as an optimization problem in the following section.

\subsection{From Local to Global Fairness}
\label{subsec:local_to_global}
We first show that \emph{enforcing an $\epsilon$–local–fairness constraint on 
\emph{each} firm automatically bounds the market‑wide opt‑out disparity.}
The result motivates our policy design: penalties can be assessed at the firm 
level, yet they control the global metric of interest.





\begin{proposition}[Local $\!\Rightarrow\!$ Global Fairness Bound]
\label{thm:local_to_global}
If every firm $j\in\{1,\dots,n\}$ satisfies the local‑fairness condition  
\[
   \max_{i,k\in\mathcal A}
   \bigl|
      \Pr(F=j\mid A=i)-\Pr(F=j\mid A=k)
   \bigr|
   \;\le\;\epsilon,
\]
then the market is $\varepsilon'$‑globally fair, i.e.  
\[
   \bigl|
      \Pr(F=0\mid A=i)-\Pr(F=0\mid A=k)
   \bigr|
   \;\le\;\varepsilon'
   \;\forall\; i,k\in\mathcal A,
\]
with
\[
   \boxed{\;
      \varepsilon'\;=\;\min\{\,n\epsilon,\,1\,\}\;
   }.
\]
\end{proposition}

\begin{proof}
For brevity write 
\(p_{j\mid i}:=\Pr(F=j\mid A=i)\).
The hypothesis gives
\(
|p_{j\mid i}-p_{j\mid k}|\le\epsilon
\)
for every firm $j\ge1$ and every pair of profiles $i,k$.

\medskip\noindent
\textbf{From local to opt‑out gap.}
Because the opt‑out probability is the complement of the in‑market mass,
\[
  p_{0\mid i}
  \;=\;
  1-\!\sum_{j=1}^{n}p_{j\mid i},
  \qquad
  p_{0\mid k}
  \;=\;
  1-\!\sum_{j=1}^{n}p_{j\mid k}.
\]
Hence
\[
\begin{aligned}
|p_{0\mid i}-p_{0\mid k}|
  &=\Bigl| \sum_{j=1}^{n}(p_{j\mid k}-p_{j\mid i}) \Bigr|
   \;\le\;\sum_{j=1}^{n}|p_{j\mid k}-p_{j\mid i}|
   \;\le\; n\epsilon.
\end{aligned}
\]

\medskip\noindent
\textbf{Probabilistic range.}
Because each $p_{0\mid\cdot}$ is a probability, their difference cannot exceed~1:
\(
|p_{0\mid i}-p_{0\mid k}|\le 1.
\)

\medskip\noindent
Combining the two bounds,
\[
   |p_{0\mid i}-p_{0\mid k}|
   \;\le\;\min\{\,n\epsilon,\,1\,\},
\]
which yields the stated $\varepsilon'$.
\end{proof}

\paragraph{Policy insight.}
Because global fairness follows directly from firm‑level constraints,
a regulator can simply penalize firms based on their own $\epsilon$–local gap; no market‑wide coordination term is needed. This concludes the motivation for our social planner's policy mechanism. In the following section, we outline the market environments in which policy explorations take place.

\section{Market Environment}
\label{sec:meth}
We model an oligopolistic market with heterogeneous consumers and profit‐maximizing firms.
We then introduce a social planner whose policy consists of a bracketed tax schedule to incentivize fairness. An illustration of this market can be found in Figure~\ref{fig:market_diagram}.
We proceed with formal definitions for our simulated agents.

\subsection{Consumers}
\label{sec:consumer_agents}
Each consumer profile \(i\) obtains utility
\[
U_{i,j} \;=\; \overline{\alpha} \;-\; \beta_i \,p_{i,j}
\]
from consuming firm \(j\)'s product, where \(\overline{\alpha}_j\) 
is base product utility, \(\beta_i\) is the price sensitivity of profile \(i\), 
and \(p_{i,j}\) is the per‐profile price. For the outside option \(F=j=0\), let 
\(U_{i,0} = \overline{\alpha}_0\). A consumer of profile \(i\) chooses among 
\(\{0,\dots,n\}\) with probability
\[
\mathbf{p}_{j\mid i}
\;=\;
\frac{\exp\bigl(U_{i,j}\bigr)}
     {\displaystyle\sum_{j=0}^{n} \exp\bigl(U_{i,j}\bigr)}.
\]

\subsection{Firms}
\label{sec:firm_agents}
Each firm \(j \in \{1,\dots,n\}\) may compute its expected profit:
\[
\mathbb{E}[\Pi_j]
\;=\;
\sum_{i=1}^{m} \mathbf{p}_{j\mid i}\,\bigl(p_{i,j} - mc_{i,j}\bigr),
\]
where \(mc_{i,j}\) is the average marginal cost for profile \(i\). Under free-market dynamics, firm \(j\)'s problem is
\[
\max_{\{p_{i,j}\}}
\;\mathbb{E}[\Pi_{j}]
\quad
\text{subject to } 0 \,\le\, p_{i,j} \,\le\, p_{\max}.
\]
Due to the inherent jump discontinuities in consumer choices under Free Market, firms solve for prices using Powell's derivative-free method~\citep{powell1964efficient} in the~\texttt{SciPy} Python optimization library~\citep{virtanen2020scipy}. While we found this to achieve better stability, we note that alternative optimization methods may be used to solve the firms' problem.

\subsection{Social Planner and Welfare Maximization}
\label{sec:planner_agent}

We now introduce a social planner who aims to maximize overall \emph{social welfare}, a hybrid measure of fairness and firm profits. To maintain policy interpretability, we also penalize large deviations from 
a simple, \emph{naive} bracketed‐tax baseline. 

\paragraph{Bracketed Fairness Tax.}
To incentivize firms to adopt fairer outcomes, we partition the fairness range \([0,1]\) into 
\(B\) brackets, each of width \(1/B\). Suppose firm~\(j\) achieves fairness \(f_j\in[0,1]\) 
and hence belongs to bracket \(b_j\), defined by 
\[
  f_j \;\in\; 
  \Bigl[\tfrac{b_j-1}{B},\;\tfrac{b_j}{B}\Bigr).
\]
We index tax brackets by \(b\in\{1,\dots,B\}\) and associate to each bracket a rate 
\(\tau_b\in[0,1]\). In general, we collect these into the vector
\[
  \boldsymbol{\tau} = [\tau_1,\dots,\tau_B]
  \;\in\; [0,1]^B.
\]
A firm in bracket~\(b_j\) thus faces an effective per‐profile margin
\(\bigl(p_{i,j} - mc_{i,j}\bigr)\,\bigl(1 - \tau_{b_j}\bigr)\). 
Consequently, under regulation, firm~\(j\)'s profit-maximization problem becomes
\[
  \max_{\{p_{i,j}\}}
  \;\mathbb{E}[\Pi_{j}]\,\bigl(1 - \tau_{b_j}\bigr),
  \quad
  p_{\min} \,\le\, p_{i,j}\,\le\, p_{\max}.
\]

\paragraph{Social Welfare Objective.}
Let \(\mathcal{W}\bigl(\boldsymbol{\tau}\bigr)\) be the total social welfare, 
\[
  \mathcal{W}\bigl(\boldsymbol{\tau}\bigr)
  \;=\;
  \Bigl(
    \frac{1}{n} 
    \sum_{j=1}^{n} 
    \mathbb{E}[\Pi_j]\!\bigl(1 - \tau_{b_j}\bigr)
  \Bigr)
  \;\times\;
  \mathrm{fairness_{global}}\!\bigl(\boldsymbol{\tau}\bigr),
\]
capturing \emph{both} global fairness (measured via the gap presented in Definition~\ref{def:global}) \emph{and} net firm profits
(under the chosen \(\boldsymbol{\tau}\)). We note that this multiplicative welfare expression is one of many possible ways to combine fairness and profit, however the intuition here is that \(\mathrm{fairness} \in [0, 1]\) effectively scales profit. A similar formulation of welfare is used in~\cite{zheng2020ai}. The planner’s goal is to select  \(\boldsymbol{\tau}\) to solve \(\max_{\boldsymbol{\tau}} \mathcal{W}\bigl(\boldsymbol{\tau}\bigr)\). To this end, we use a soft actor-critic algorithm~\citep{softactorcritic} to train an RL agent whose reward is \(\mathcal{W}\bigl(\boldsymbol{\tau}\bigr)\).

\noindent In Algorithm~\ref{alg:price-game}, we outline the simultaneous Nash Free Market in which firms select prices to optimize for profit given a policy generated by the social planner.

\begin{algorithm}[!ht]
\caption{Multi-Agent Price-Setting Game}
\label{alg:price-game}
\begin{algorithmic}[1]

\renewcommand{\algorithmiccomment}[1]{\hfill $\triangleright$ #1}

\STATE {\bf Input:}
\STATE \quad Number of firms $n$ (indexed by $j$); consumer profiles $i=1,\dots,m$ with size $S_i$ and price sensitivity $\beta_i$;
\STATE \quad Base utility $\overline{\alpha}_j$ for each firm $j$; outside option utility $\overline{\alpha}_0$;
\STATE \quad Tax brackets $\boldsymbol{\tau} = [\tau_1, \dots, \tau_B]$ set by planner; marginal costs $mc_{i,j}$.

\STATE {\bf Initialize:}
\STATE \quad Each firm $j$ has prices $p_{i,j}$ (possibly random or previously set).

\vspace{0.5em}
\STATE {\bf Step 1: Social Planner Sets Tax Policy}
\FOR{$j \leftarrow 1 \text{ to } n$}
  \STATE Compute fairness $f_j$ for firm $j$
  \STATE Assign bracket $b_j \leftarrow \text{bracket index such that } f_j \in \left[\tfrac{b_j-1}{B}, \tfrac{b_j}{B}\right)$
  \STATE Set effective tax rate $\tau_{b_j}$ for firm $j$
\ENDFOR
\STATE \quad \COMMENT{Firm $j$'s margin becomes $(p_{i,j} - mc_{i,j})(1 - \tau_{b_j})$}

\vspace{0.5em}
\STATE {\bf Step 2: Firms Simultaneously Update Prices}
\FOR{$j \leftarrow 1 \text{ to } n$}
  \STATE $p_{i,j} \leftarrow \arg\max_{p_{i,j}} \sum_{i=1}^{m} \mathbf{p}_{j \mid i} \cdot (p_{i,j} - mc_{i,j})(1 - \tau_{b_j}) \cdot S_i$
  \STATE \COMMENT{Firm $j$ chooses prices to maximize expected profit}
\ENDFOR

\vspace{0.5em}
\STATE {\bf Step 3: Consumers Choose Firm or Outside Option}
\FOR{$i \leftarrow 1 \text{ to } m$}
  \FOR{$j \leftarrow 1 \text{ to } n$}
    \STATE $U_{i,j} \leftarrow \overline{\alpha}_j - \beta_i \, p_{i,j}$
  \ENDFOR
  \STATE $U_{i,0} \leftarrow \overline{\alpha}_0$
  \STATE $\mathbf{p}_{j \mid i} \leftarrow \frac{\exp(U_{i,j})}{\exp(U_{i,0}) + \sum_{k=1}^{n} \exp(U_{i,k})} \quad \forall j$
\ENDFOR

\vspace{0.5em}
\STATE {\bf Step 4: Outcome and Payoffs}
\FOR{$j \leftarrow 1 \text{ to } n$}
  \STATE $\text{Demand}_{i,j} \leftarrow S_i \cdot \mathbf{p}_{j \mid i} \quad \forall i$
  \STATE $\Pi_j \leftarrow \sum_{i=1}^{m} (p_{i,j} - mc_{i,j})(1 - \tau_{b_j}) \cdot \text{Demand}_{i,j}$
\ENDFOR

\vspace{0.5em}
\STATE {\bf Step 5: End of Game}
\STATE \quad Output final $\boldsymbol{\tau}$, prices $\{p_{i,j}\}$, demands, and profits $\{\Pi_j\}$.

\end{algorithmic}
\end{algorithm}

\begin{figure*}[!ht]
  \centering
  \includegraphics[scale=0.60]{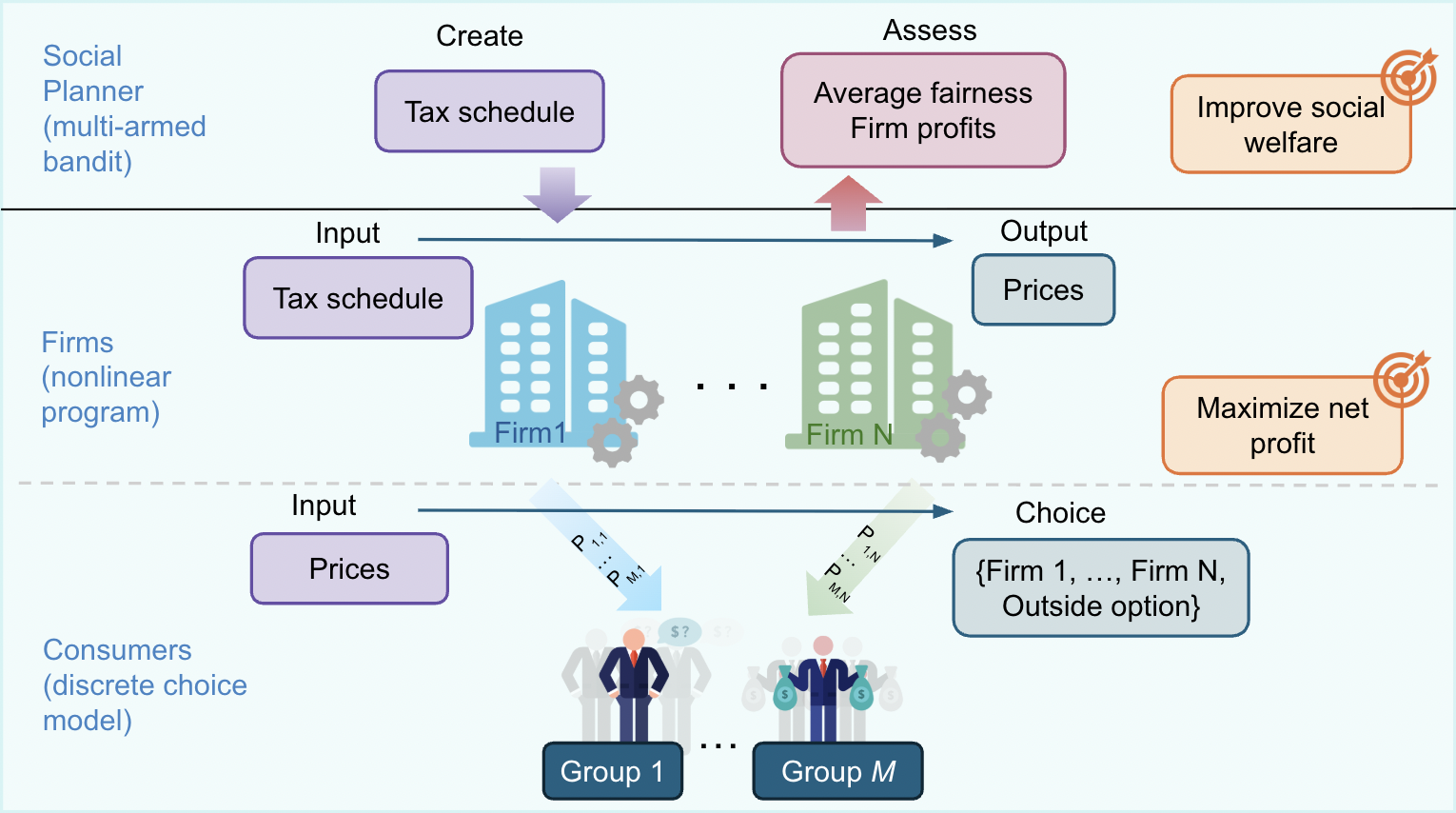}
  \vspace{5pt}
  \caption{A dynamic-pricing market consisting of 3 agent types, each with their own optimization objective. The social planner generates welfare-maximizing tax schedules applied to firms based on their local fairness gap. Firms then compute their best responses and assign consumer group-level prices. Finally, consumers make their selection from these prices.}
  \label{fig:market_diagram}
\end{figure*}

\section{Market Parameterization and Empirical Results}
\label{sec:dyn_incentives}

We consruct two market environments, \emph{health insurance} and \emph{consumer lending}, because both combine risk-based pricing with pronounced distributional concerns.  
Income‐group proportions follow \citet{pew2024middleclass}; insurance-coverage rates draw on \citet{keisler2024health}; and home-ownership patterns (a proxy for credit demand) follow \citet{census_homeownership_2020}.  
The population is divided into High, Middle, and Low income segments, each assigned a price elasticity~(\(\beta\)) and a firm‐specific marginal cost~(\(mc\)).
These parameters, which can be found for both markets in table~\ref{tab:market_init}, introduce system-wide heterogeneity.

\paragraph{Why heterogeneity matters.}
Heterogeneity arises among both agent participant types in the market and shapes every dimension of the policy problem:

\begin{itemize}
\item \textbf{Consumers.}  Differences in disposable income, outside options, and risk exposure create a spread of price elasticities.  In insurance, demand is relatively inelastic at higher incomes because coverage quality is valued more than marginal dollars.  In credit, by contrast, wealthier households can leverage their assets to negotiate better terms, pay cash or source cheaper capital altogether, making them more price-sensitive.  Such a cross-market reversal broadly illustrates how elasticity is a joint outcome of preference intensity and available substitutes.
\item \textbf{Firms.}  
Marginal cost is tightly linked to borrowers’ or policyholders’ income.
Low-income consumers are generally riskier to serve as they face higher job volatility, have thinner financial buffers~\citep{bertoletti2024consumer, fout2020credit}, and—in the case of health insurance—experience more occupational hazards and reduced access to preventive care~\citep{nicholson2008disparities}.
These factors translate into (i) higher expected claim costs for insurers and (ii) elevated default probabilities for lenders, raising the average per-unit cost of coverage or credit.  
By contrast, high-income consumers offer steadier cash flows, better health profiles, and superior collateral, enabling firms to price at lower cost.  
Even within a single industry, providers differ in their ability to manage this risk—large insurers leverage pooled data and predictive analytics, whereas smaller or niche firms often specialize in higher-risk pools—amplifying cost dispersion and strategic asymmetry.

\end{itemize}
From a regulatory standpoint, it is important to consider both sources of heterogeneity in order for policies to achieve realistic social welfare improvements. Our bracketed tax is therefore designed to be \emph{piecewise}—simple enough for transparency yet flexible enough to align marginal incentives across diverse firms. We compare three baselines:
\begin{itemize}
\setlength{\itemsep}{3pt}
\item \textbf{Free Market:} Firms compete in a simultaneous price-setting game with the aim of maximizing their individual profits in the absence of policy intervention.
\item \textbf{Linear regulation:} a monotonic, bracketed linear tax
\[
  \tau^\mathrm{base}_b 
  \;=\;
  1 \;-\; \frac{b}{B},
  \quad
  \text{for } b=1,\dots,B.
\]
      intended as a hand-crafted fairness correction. Intuitively, this baseline approximates a simple rule (\(\tau = (1-\!\textit{fairness})\)),
discretized into \(B\) brackets.

\item \textbf{Collusion:} Firms jointly maximize aggregate profit. This benchmark reveals the upper bound on total profitability under perfect coordination, and can be seen as an oracle case for profit.
\end{itemize}
For these benchmarks, we report outcomes after convergence to Nash equilibrium prices.

\begin{table}[ht]
\centering
{\scriptsize
\setlength{\tabcolsep}{3pt}
\renewcommand{\arraystretch}{1.05}
\begin{tabular}{@{}l c ccc !{\vrule width 0.8pt} cccccc@{}}
\toprule
        &      & \multicolumn{3}{c}{\textbf{Insurance}}
        & \multicolumn{6}{c}{\textbf{Credit}}\\
\cmidrule(lr){3-5}\cmidrule(lr){6-11}
\textbf{Group} & \textbf{N}
  & $\beta$ & $mc_1$ & $mc_2$
  & $\beta$ & $mc_1$ & $mc_2$ & $mc_3$ & $mc_4$ & $mc_5$ \\
\midrule
High (H)   & 200 & 0.25  & 2.50 & 2.25 & 3.00 & 0.40 & 0.65 & 0.45 & 0.60 & 0.44 \\
Middle (M) & 520 & 0.70  & 3.00 & 2.75 & 2.70 & 1.20 & 1.45 & 1.12 & 1.35 & 1.29 \\
Low (L)    & 280 & 0.825 & 3.50 & 3.25 & 2.25 & 2.05 & 2.30 & 2.25 & 2.28 & 2.10 \\
\bottomrule
\end{tabular}
\caption{Baseline demand elasticities and marginal costs. The insurance market is modelled with two competing insurers, and the credit market with five lenders. Price bounds are $P_{\min}=1$, $P_{\max}=20$.}
\label{tab:market_init}
}
\end{table}

\begin{table}[ht]
\centering
{\scriptsize
\begin{tabular}{@{}l c c c c@{}}
\rowcolor{gray!15}
\multicolumn{5}{c}{\textbf{Shared social-planner parameters}}\\
\toprule
Algorithm & Brackets $B$ & $\tau_{\min}$ & $\tau_{\max}$
          & $\lambda_{\text{ins}}/\lambda_{\text{cred}}$\\
\midrule
SAC & 20 & 0\% & 100\% & 100 \;/\; 10\\
\bottomrule
\end{tabular}
\caption{Social-planner initialization parameters.}
\label{tab:sac_init}
}
\end{table}

\subsubsection{Demand elasticity.}
Elasticity captures both willingness and ability to substitute.  For health insurance, the absence of close substitutes renders wealthier consumers less price-responsive.  In credit, abundant alternatives (e.g.\ home-equity lines, credit cards, or abstention) make the same group more price-elastic, whereas low-income borrowers confront a near take-it-or-leave-it contract.

\subsubsection{Marginal cost.}
Risk-adjusted cost falls with income in both markets but for distinct reasons: fewer costly medical claims in insurance, and lower default probabilities in lending. Firms, heterogeneous in ressources and capability, modulate these costs further.
\newline Finally, the social planner learns a piecewise-constant tax via Soft Actor–Critic (Table~\ref{tab:sac_init}).  By explicitly targeting cross-segment disparities, the learned policy reflects heterogeneity on \emph{both} the consumer and the firm side, in contrast to the naïve linear schedule.

\paragraph{\(\ell_1\) Penalty on Deviation from Baseline.}
To retain policy interpretability, we penalize the \emph{planner} 
for learning a tax schedule \(\boldsymbol{\tau}\) that deviates 
excessively from the naive baseline. Specifically, we add an \(\ell_1\) regularizer
\[
  \lambda \sum_{b=1}^B 
  \Bigl|\;\tau_b - \tau^\mathrm{base}_b\Bigr|,
\]
where \(\lambda \ge 0\) tunes how much the planner is incentivized to remain close
to \(\tau^\mathrm{base}\). Thus, the planner tweaks the naive schedule only to the
extent that it increases overall social welfare. This can be interpreted as an ``expert planner" capable of taking a known policy mechanism understood by domain experts and adapting it to specific markets. The social planner's full optimization problem is therefore
\[
  \max_{\boldsymbol{\tau}}
  \Biggl[
    \mathcal{W}\bigl(\boldsymbol{\tau}\bigr)
    \;-\;
    \lambda 
    \sum_{b=1}^B 
    \bigl|\tau_b - \tau_b^\mathrm{base}\bigr|
  \Biggr].
\]
Hence, while large deviations from the baseline \(\tau^\mathrm{base}\) are penalized, 
the planner retains the flexibility to adjust tax rates to increase overall social welfare 
(e.g., by nudging firms toward more equitable pricing where the naive schedule is suboptimal).

\subsection{Case Study 1: Health Insurance}
\subsubsection{Market Overview and Motivation.} Census data report 7.9\% of the U.S.\ working population uninsured ($\sim$25M). Though law-abiding insurance providers do not explicitly use sensitive attributes in the determination of premia, these can often be inferred from occupation and ZIP code~\citep{dwork2012fairness,barocas2016big}, making health insurance an ideal market for dynamic regulation. As mentioned, this instance of our simulation represents a market for a necessity good, and thus we endow low-risk profiles with low demand elasticity, and high-risk profiles with high demand elasticity.

\subsubsection{Policy Generation and Empirical Results.}
Use of our method to regulate the health insurance market yields substantial welfare gains over benchmarks. We measure \textbf{net profit} (after tax), \textbf{fairness}, and \textbf{social welfare} at both the per‐firm and global level. Table~\ref{tab:comparison_markets} (\textit{left}) shows the outcomes under four market regimes: (\textbf{Free Market}) purely competitive, (\textbf{Linear-SP}) a simplistic bracket policy with fixed cutoffs, (\textbf{RL-SP}) our proposed social planner, and (\textbf{Collusion}), a stable cartel‐like coordination (used solely for revealing an upper bound on profit from which we provide normalized profit values for the benchmarks of interest). From results, we note that the RL method achieved the highest social welfare, not only by improving fairness, but also in improving total market profitability, indicating that it was able to mitigate the social dilemma inherent to competitive games, instead forcing firms into a kind of policy-induced tacit cooperation that makes the market more stable and friendly to consumers. This is evidence that a fairness-seeking policy can outperform selfish Nash firms in terms of aggregate profits, suggesting that the RL policy might act as an implicit coordination device, getting firms closer to the cooperative outcome without explicit collusion. Broadly, our RL social planner improved social welfare in the market for health insurance compared to the baseline competitive case by approximately 11\%, and outperformed the linear policy by 10\%. After 2M training steps, the RL social planner converged to the policy found in Figure~\ref{fig:insurance_policy_generation}, with which, as seen in Table~\ref{tab:comparison_markets} (\textit{left}), it was able to funnel firms into welfare-improving fairness brackets. Further, providing the agent with an intuitive prior allowed it to maintain interpretability in brackets with few or no training examples. Specifically, it tweaked the Linear-SP baseline in fairness areas where it was able to achieve substantial welfare gains. 
\begin{figure}[!ht]
\centering
\begin{subfigure}{0.48\textwidth}
    \centering
    \includegraphics[width=\textwidth]{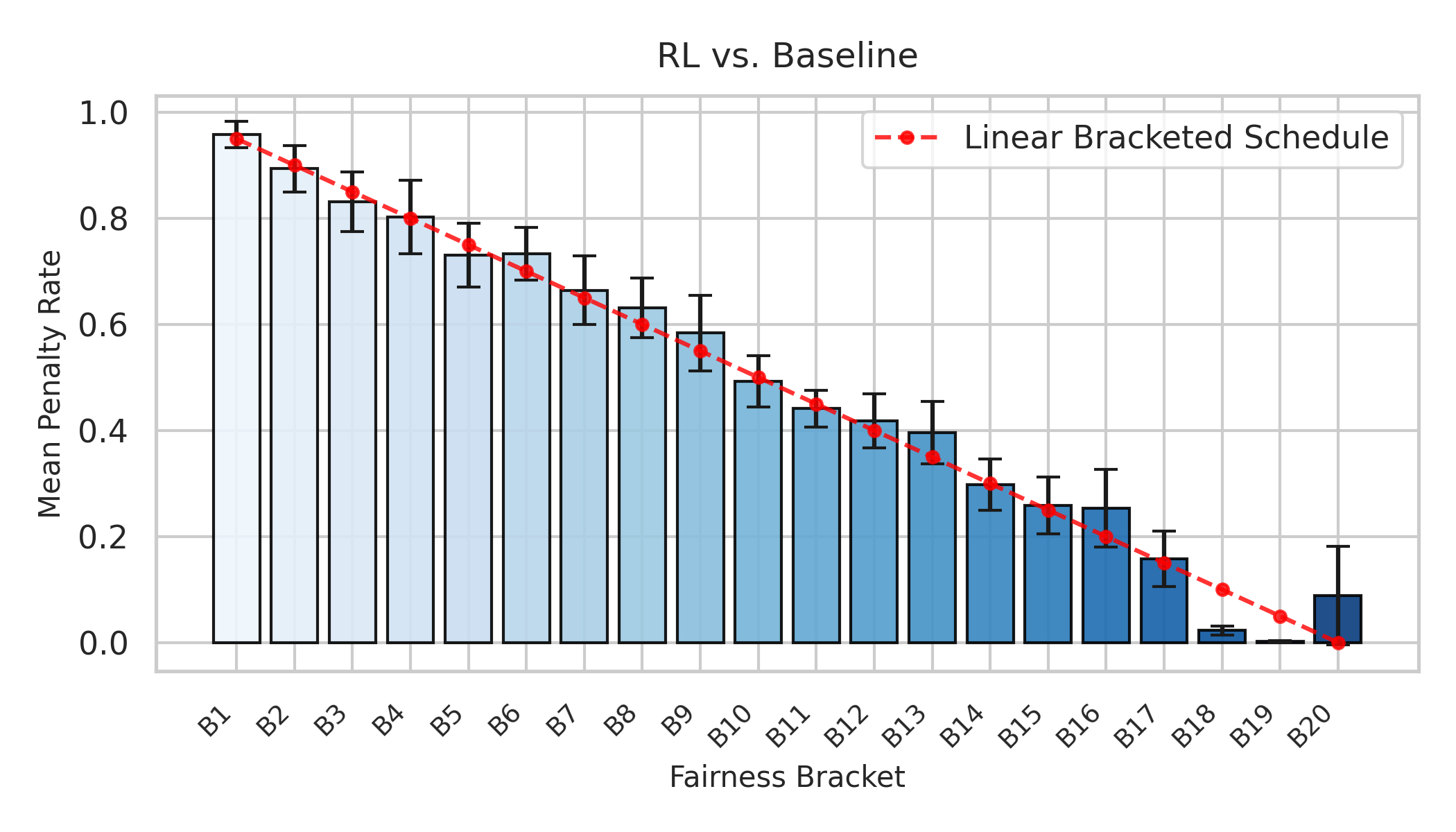}
    \vspace{-0.75cm}
    \caption{Taxation policy generated by the RL social planner for the insurance market.}
    \label{fig:insurance_policy_generation}
\end{subfigure}
\begin{subfigure}{0.48\textwidth}
    \centering
    \includegraphics[width=\textwidth]{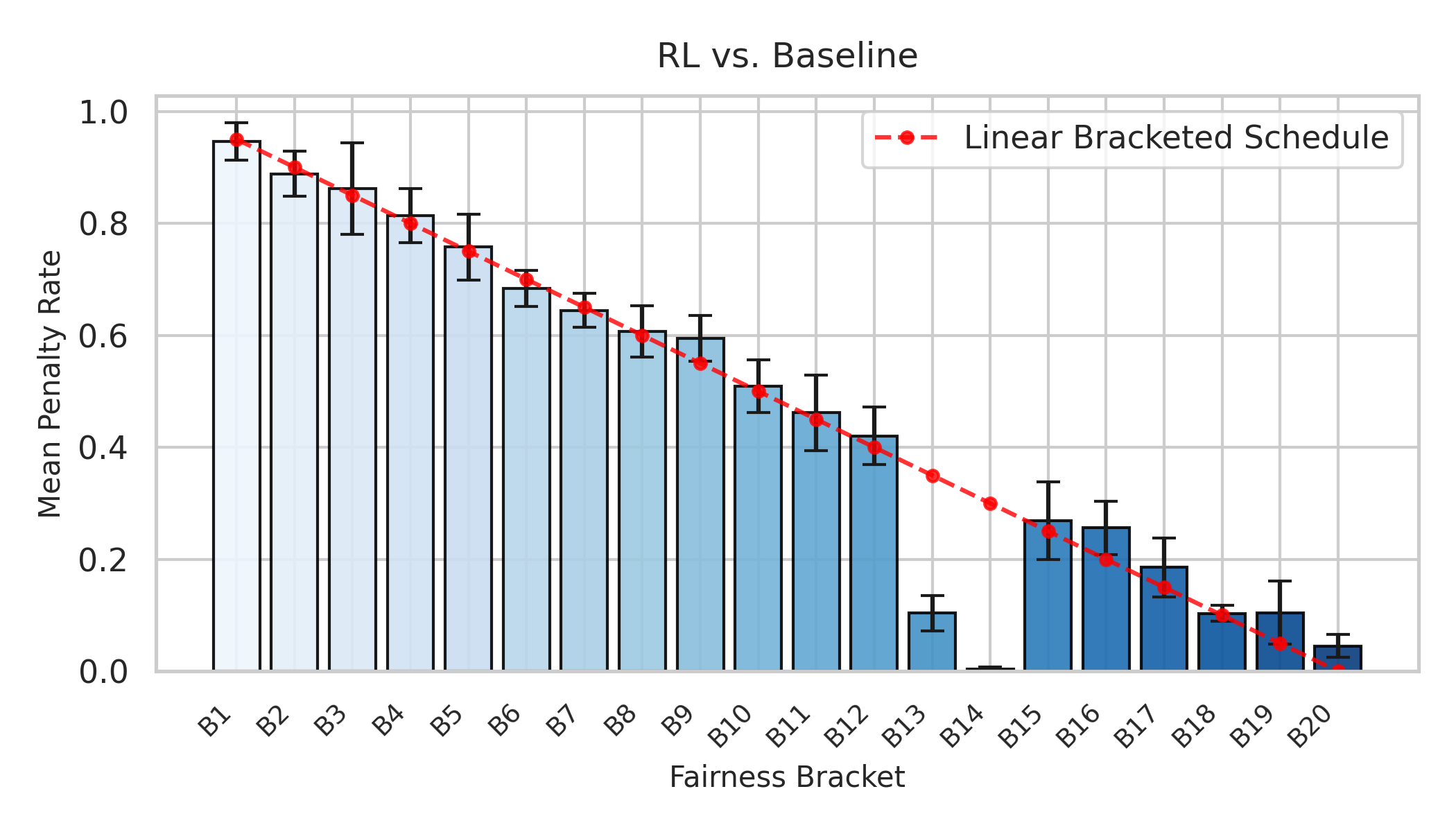}
    \vspace{-0.75cm}
    \caption{Taxation policy generated by the RL social planner for the credit market.}
    \label{fig:credit_policy_generation}
\end{subfigure}
\vspace{0.5cm}
\caption{Comparison of policy generation across two different markets:
Insurance (a) and Credit (b).}
\label{fig:policy_generation_both}
\end{figure}

\setlength\dashlinegap{1pt}   
\setlength\dashlinedash{2pt}  

\begin{table*}[!ht]
\centering
\small  
\setlength{\tabcolsep}{5pt}  
\begin{tabular}{l ccc:c | ccc:c}
\toprule
& \multicolumn{4}{c}{\textbf{Insurance}} 
& \multicolumn{4}{c}{\textbf{Credit}} \\
\cmidrule(lr){2-5}
\cmidrule(lr){6-9}
& \textbf{Free Market} & \textbf{Linear-SP} & \textbf{RL-SP} & \textbf{Collusion}
& \textbf{Free Market} & \textbf{Linear-SP} & \textbf{RL-SP} & \textbf{Collusion} \\
\midrule
\textbf{Profit $\uparrow$} 
  & 0.697 & 0.642 & 0.707 $\pm$ 0.002 & \textit{1.0} 
  & 0.630 & 0.551 $\pm$ 0.003 & 0.622 $\pm$ 0.003 & \textit{1.0} \\
\textbf{Fairness $\uparrow$}
  & 0.821 & 0.895 & 0.895 & \textit{0.851}
  & 0.660 & 0.712 $\pm$ 0.006 & 0.767 $\pm$ 0.004 & \textit{0.709} \\
\textbf{Opt Out $\downarrow$}
  & 0.137 & 0.120 & 0.121 & \textit{0.335}
  & 0.173 & 0.1593 $\pm$ 0.002 & 0.218 $\pm$ 0.002 & \textit{0.395} \\
\midrule
\textbf{Welfare}
  & 0.572 & 0.575 & \textbf{0.633 $\pm$ 0.002} & \textit{0.851}
  & 0.416 & 0.392 $\pm$ 0.004 & \textbf{0.477 $\pm$ 0.003} & \textit{0.709} \\
\bottomrule
\end{tabular}
\caption{Comparison of Profit, Fairness, and Welfare across market scenarios and dynamics. We include market-wide opt-out rates to broadly outline market outcomes. Profit values are normalized with respect to the theoretical maximum determined by the collusive (oracle) market setting. We report standard errors (SE) over 5 seeds, and omit SE values under 0.001.}
\label{tab:comparison_markets}
\end{table*}

\subsection{Case Study 2: Consumer Credit}
\label{sec:credit}
\subsubsection{Market Overview and Motivation.}
A second key application of our framework pertains to consumer credit, where financial institutions offer loans or lines of credit to heterogeneous borrowers through competitive interest rates. Unlike health insurance, credit is generally less “essential” yet the stakes remain high for consumers with limited collateral or volatile income. For these groups, restricted credit access can hamper opportunities to secure home equity, reinforcing wealth gaps. As such, credit markets present a core tension between risk-based pricing (necessary for profitable lending) and equitable access (necessary for social welfare and fairness), making them an ideal testing ground for dynamic regulation.

\subsubsection{Policy Generation and Empirical Results.}
We evaluate the outcomes of our credit market simulation across the same four market regimes, though in this instance each assessed over five firms. In this market, we observe a qualitatively similar dynamic to the insurance scenario: the RL social planner consistently improves upon linear policy interventions and competitive baselines, even in a setting with greater firm heterogeneity and tighter fairness-profit tradeoffs. Evaluated over five competing lenders, Table~\ref{tab:comparison_markets} (\textit{right}) shows that the competitive market yields high profit levels (0.630) but suffers from low fairness (0.660). The Linear-SP baseline improves fairness to 0.712 but imposes a blunt restriction on risk-based interest rates, thereby driving down profits and yielding modest improvements in participation. By contrast, the RL policy exhibits adaptive behaviour, learning market-specific bracket assignments that improve fairness further (0.767) while also maintaining profitability approaching that of the Free Market baseline. However, we note an increase in the global opt-out rate under the RL policy, a result discussed further in the next section. Nonetheless, overall social welfare increases to 0.477, reversing the decline observed under the Linear-SP baseline (0.392). This represents social welfare gains of 15\% and 22\% over the competitive and Linear-SP baselines, respectively.
\\
\noindent
Altogether, these results confirm that an RL‐based regulator can foster an equilibrium where profitability and fairness minimally conflict.

\medskip
\section{Discussion}
A notable aspect of our experiments is how tacit cooperation emerges within certain policy frameworks, evidenced by how markets under the RL policy yield comparable or even higher profits than under unregulated Free Market. This is indicative of a social dilemma in unregulated Free Market: profit-maximizing individual firm behavior reduces aggregate profits. Meanwhile, by coordinating firm incentives, the regulator resolves this inefficiency, while simultaneously aligning firm and consumer welfare. Crucially, this “cooperation” among firms is driven not by explicit collusion, but instead by an understanding of specific market dynamics acquired by the social planner, allowing it to design incentives which mirror cooperation. From a policy perspective, these insights suggest that partial regulation in the form of adaptive bracket constraints that do not artificially cap or flatten prices can be more attractive than rigid price bounds or unregulated Free Market. This approach preserves profitability for firms while simultaneously broadening access and mitigating discriminatory pricing. In practice, such bracket tuning could be made transparent to both firms and policymakers, enabling oversight of how fairness categorizations evolve over time.

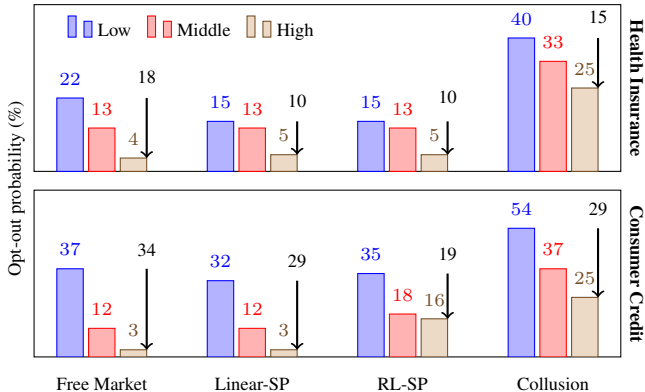
\begin{figure}[ht]
\centering
\begin{tikzpicture}
\begin{groupplot}[
    group style={
        group size=1 by 2,
        vertical sep=0.25cm
    },
    ybar,
    x=2.0cm,                 
    width=0.48\textwidth,
    height=3.8cm,
    enlarge x limits=0.15,    
    symbolic x coords={Free Market,Linear-SP,RL-SP,Collusion},
    xtick=data,
    xticklabel style={anchor=north},
    every node near coord/.append style={
        font=\scriptsize,
        anchor=south,
        yshift=1.5pt
    },
    ticklabel style={font=\scriptsize},
    yticklabels={},
    tick style={draw=none},
]

\nextgroupplot[
    ymin=0, ymax=50,
    nodes near coords,
    xticklabels={},
    legend style={
        font=\scriptsize,
        at={(0.5,0.97)}, anchor=north east,
        legend columns=4,
        /tikz/every even column/.append style={column sep=4pt},
        draw=none, fill=none
    },
]
\addplot coordinates {(Free Market,22) (Linear-SP,15) (RL-SP,15) (Collusion,40)};
\addplot coordinates {(Free Market,13) (Linear-SP,13) (RL-SP,13) (Collusion,33)};
\addplot coordinates {(Free Market, 4) (Linear-SP, 5) (RL-SP, 5) (Collusion,25)};
\begin{scope}[xshift=17pt]
\draw[->, thick] (axis cs:Free Market,22) -- (axis cs:Free Market,4);
\node[font=\scriptsize, anchor=south] at (axis cs:Free Market,23.5) {18};

\draw[->, thick] (axis cs:Linear-SP,15) -- (axis cs:Linear-SP,5);
\node[font=\scriptsize, anchor=south] at (axis cs:Linear-SP,16.5) {10};

\draw[->, thick] (axis cs:RL-SP,15) -- (axis cs:RL-SP,5);
\node[font=\scriptsize, anchor=south] at (axis cs:RL-SP,17.5) {10};

\draw[->, thick] (axis cs:Collusion,40) -- (axis cs:Collusion,25);
\node[font=\scriptsize, anchor=south] at (axis cs:Collusion,41.5) {15};
\end{scope}

\legend{Low, Middle, High, $\Delta_{\max}$}

\nextgroupplot[
    ymin=0, ymax=70,
    nodes near coords,
]
\addplot coordinates {(Free Market,37) (Linear-SP,32) (RL-SP,35) (Collusion,54)};
\addplot coordinates {(Free Market,12) (Linear-SP,12) (RL-SP,18) (Collusion,37)};
\addplot coordinates {(Free Market,3) (Linear-SP,3) (RL-SP,16) (Collusion,25)};
\begin{scope}[xshift=17pt]
\draw[->, thick] (axis cs:Free Market,37) -- (axis cs:Free Market,3);
\node[font=\scriptsize, anchor=south] at (axis cs:Free Market,38.5) {34};

\draw[->, thick] (axis cs:Linear-SP,32) -- (axis cs:Linear-SP,3);
\node[font=\scriptsize, anchor=south] at (axis cs:Linear-SP,33.5) {29};

\draw[->, thick] (axis cs:RL-SP,35) -- (axis cs:RL-SP,16);
\node[font=\scriptsize, anchor=south] at (axis cs:RL-SP,36.5) {19};

\draw[->, thick] (axis cs:Collusion,54) -- (axis cs:Collusion,25);
\node[font=\scriptsize, anchor=south] at (axis cs:Collusion,55.5) {29};
\end{scope}

\end{groupplot}

\node[
    rotate=90,
    anchor=south,
    font=\scriptsize,
    xshift=-0.1cm
] at ($(group c1r1.west)!0.5!(group c1r2.west)$)
  {Opt‑out probability (\%)};

\node[
    rotate=270,
    anchor=south,
    font=\scriptsize
] at ($(group c1r1.east)$) {\textbf{Health Insurance}};

\node[
    rotate=270,
    anchor=south,
    font=\scriptsize
] at ($(group c1r2.east)$) {\textbf{Consumer Credit}};

\end{tikzpicture}
\caption{Mean consumer‑group opt‑out rates under multiple market frameworks. The arrow indicates the maximum difference in opt-out rate means between population groups. By global fairness, lower is better.}
\label{fig:opt_out_barplot}
\end{figure}

\noindent Despite obtaining welfare improvements with our RL method, we deem it necessary to highlight a shortcoming regarding our fairness criterion. While demand fairness advocates for equal access to goods (as does \emph{demographic parity}, it is insufficient on its own at ensuring that consumer participation rates improve, as it simply requires that they be agnostic to group membership, with no weight accorded to their actual values. From Table~\ref{tab:comparison_markets}, we find that fairness under \textbf{collusion}, an illicit market practice in most jurisdictions, actually outperforms Free Market in both markets, despite yeilding the highest opt out rates by a substantial margin. This is not because it truly encourages participation among underrepresented groups, but because it simply makes all consumers more equally \emph{un}likely to participate. This is evident from Figure~\ref{fig:opt_out_barplot}, where we report group-level opt-out rates. From these, we observe that this share is higher for each subgroup under collusion. We note further that, under the RL policy, there are reductions in opt-out rates among the low-income group in both markets compared to Free Market, albeit at the cost of an opt-out rate increase among the high and middle-income groups. Thus, it is important to examine resulting market distributions in order to qualitatively assess the impact of policy generations. Broadly, while equal access is a desirable goal, it should not come in the form of increased exclusion. A fairness criterion that improves parity in access but leads to a larger number of individuals being priced out or discouraged from participation can ultimately exacerbate structural harms and decrease overall social welfare. To alleviate this concern, one might consider alternate metrics, such as consumer surplus fairness~\cite{cohen2021dynamic}, which, when using our probabilistic setting, can be defined via the log-sum rule~\cite{small1981applied} as
\[
\bigl|\text{max}_i\;CS(\boldsymbol{\tau})_{i} - \text{min}_k\;CS(\boldsymbol{\tau})_{k}\bigr|\;\leq\;\epsilon, \forall\;{i, k},
\]
where
\[
CS(\boldsymbol{\tau})_{i} = \frac{1}{\beta_i}\text{log}\Bigl(\sum_j^N\;\text{exp}(\overline{\alpha}_{j} - \beta_ip_{i,j})\Bigr).
\]
Accordingly, \texttt{MarketSim} makes it straightforward to swap any welfare metric desired by the user and evaluate outcomes in terms of firm performance and consumer distribution.

\subsection{Ablation Study: Market Bounds on Fairness}
\label{sec:ablation}
While fairness, according to our local and global definitions, has a theoretical upper bound of $1.0$, the empirical upper bound is largely dependent upon market dynamics and parameter initializations.
\begin{table}[ht]
\centering
\begin{tabular}{c c c}
\toprule
& Insurance & Credit \\
\midrule
Firm A & 0.947 & 0.823\\
Firm B & 0.947 & 0.885 \\
Firm C & - & 0.793 \\
Firm D & - & 0.902 \\
Firm E & - & 0.763 \\
Global & 0.895 & 0.791 \\
\bottomrule
\end{tabular}
\caption{Local and global fairness values under a fairness-maximizing SP tax policy.}
\label{fig:ablation_fairness_bound}
\end{table}
Thus, to uncover the empirical upper bounds on fairness in both markets, we train a new social planner with the sole objective of maximizing global demand fairness with no weight on profit. Results reveal that our instance of the market for insurance has an empirical upper bound on global fairness of $0.895$, which was achieved by both Linear-SP baselines and welfare-maximizing RL policies. Meanwhile, in the credit market, our fairness-maximizing SP reveals an empirical upper bound of $0.791$, which was approached but never achieved by our welfare-maximizing SP. This can influence the degree of improvement one can expect when experimenting with regulatory frameworks in \texttt{MarketSim}.

\section{Runtime \& Scalability of \texttt{MarketSim}}
\label{sec:runtime}
We explore runtimes to demonstrate the approximately-linear scalability of price Free Market in \texttt{MarketSim}, reporting wall-clock convergence times on markets of 2 to 100 firms. As the number of participating firms may vary between markets, our simulation should scale well over a broad range of firm counts.
\begin{figure}[ht]
\centering
\begin{tikzpicture}
\begin{axis}[
    width=0.9\linewidth,
    height=5.2cm,
    xlabel={Number of Firms},
    ylabel={Mean Runtime (s)},
    xmin=0, xmax=105,
    ymin=0, ymax=5,
    xtick={0,20,40,60,80,100},
    ytick={0,1,2,3,4,5},
    axis lines=left,
    enlarge x limits=0.05,
    enlarge y limits=0.05,
    grid=both,
    grid style={gray!15},
    major grid style={gray!30},
    tick style={gray!70},
    tick label style={font=\footnotesize},
    label style={font=\footnotesize},
    line width=1pt,
    mark size=.5pt,
    error bars/y dir=both,
    error bars/y explicit,
    every axis plot/.append style={thick},
    clip=false,
]

\addplot[
    color=blue!60!black,
    mark=*,
    mark options={fill=blue!40},
    error bars/.cd,
    y explicit,
]
coordinates {
(2, 0.096127) +- (0, 0.003554)
(20, 0.823836) +- (0, 0.130398)
(30, 1.158854) +- (0, 0.036261)
(40, 1.684572) +- (0, 0.190766)
(50, 2.118188) +- (0, 0.243765)
(60, 2.493393) +- (0, 0.274698)
(70, 2.908767) +- (0, 0.173760)
(80, 3.455010) +- (0, 0.213952)
(90, 3.964190) +- (0, 0.318150)
(100, 4.352400) +- (0, 0.297467)
};
\end{axis}
\end{tikzpicture}
\caption{Mean wall-clock runtime per 10 rounds as a function of the number of firms (on CPU). Error bars represent standard deviation over 5 seeds.}
\label{fig:runtime_errorbars}
\end{figure}
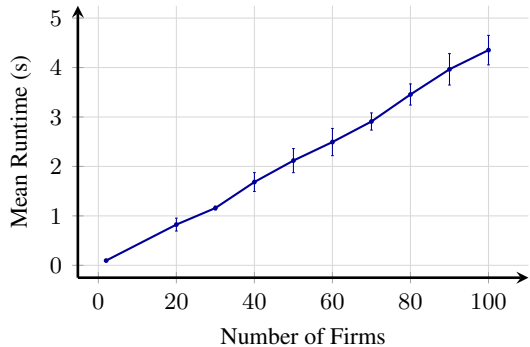
For instance, 80\% the U.S. health insurance market is dominated by 10 firms, with 57\% concentrated in the top 3~\cite{ama2023market}. Meanwhile, the credit market is more fragmented, with many non-bank lenders making up 65\% of the mortgage lending market~\cite{cfpb2024nonbank}. Conversely, ~70\% of the credit card market is dominated by 5 issuers~\cite{wallethub2025creditcards}, indicating a broader range for participation within lending markets.~\texttt{MarketSim} nonetheless accommodates arbitrarily many firms and consumers, with runtimes scaling linearly.

\subsubsection{Limitations and Future Work.}
While these findings offer promising evidence, we highlight several directions for further inquiry:

\begin{enumerate}
    \item \textbf{Multi-step Taxation Dynamics:} While this work focused on a single-step taxation policy, an exciting direction for future research is to frame taxation as a sequential decision-making task. Such settings could further emphasize the advantages of RL approaches.
    \item \textbf{Data Quality and Bias:} Real‐world transactional datasets often contain inaccuracies, missing fields, or historical biases, which could skew bracket assignments. Investigating robustness under such conditions remains an essential next step.
    \item \textbf{Bracket Design:} We implement a straightforward bracket structure here. Future work might explore more sophisticated, flexible brackets or alternative reward functions that accommodate multiple fairness definitions or risk preferences.
    \item \textbf{Legal and Ethical Context:} As bracketed interventions shape competitive behavior at scale, deeper analysis is warranted to confirm compatibility with antitrust laws and to guard against new forms of collusion or bias.
\end{enumerate}

\noindent In summary, our case studies demonstrate that dynamic bracket regulation can notably enhance fairness and welfare in a competitive market scenario feature dynamic pricing with heterogeneous firms and consumer profiles. By aligning private incentives with public interest goals, RL‐driven brackets exemplify a viable pathway toward more inclusive market systems for essential products and services.

\section{Conclusion}
In this work, we explore social welfare outcomes under various policy mechanisms in markets featuring risk-based dynamic pricing. We mathematically demonstrate the relationship between local and global notions of fairness, thereby motivating a firm-level policy approach. We then introduce \texttt{MarketSim}, a simulator for risk-based dynamic pricing, which we render open-source. Finally, we reproduce two distinct real world markets where global opt-out gaps have been empirically recorded (health insurance and consumer credit) and demonstrate how RL can be successfully leveraged to generate interpretable policy mechanisms aimed at improving social welfare.

\section{Ethical Statement}
If rigorously validated and overseen, AI-assisted regulation could help align private incentives with societal fairness goals in essential-goods markets (e.g.\ health insurance, consumer credit).  Conversely, careless application risks legitimising opaque pricing schemes and deepening existing disparities.  We urge future work to prioritise transparency, multidisciplinary oversight, and the voices of affected communities.

\section*{Acknowledgements}
Funding support for project activities has been partially provided by Canada CIFAR AI Chair, Google award, and FRQNT scholarships.
We also express our gratitude to Compute Canada and Mila clusters for their support in providing facilities for our evaluations.

\bibliography{bibliography}

\appendix

\section{Price Convergence in Insurance Simulation}
\label{appendix:price_convergence}
Firms converge to different price allocations dependent upon the market dynamics. Under fairness incentives (linear and RL policies), firms are nudged to adopt lower prices for the low income group than under Free Market. When firms collude, they set prices much higher to maximize total market profits.
\definecolor{myGreenL}{HTML}{C9E7B1} 
\definecolor{myGreenM}{HTML}{A3CC85} 

\definecolor{myRedL}{HTML}{F7C1BB}   
\definecolor{myRedM}{HTML}{EE8572}   

\definecolor{myBlueL}{HTML}{A7D7E8}  
\definecolor{myBlueM}{HTML}{63B2CF}  

\begin{tikzpicture}
\begin{axis}[
    width=0.48\textwidth, 
    height=4.5cm,
    ybar,
    bar width=4pt,
    enlarge x limits=0.15,
    ylabel={Premia},
    symbolic x coords={Free Market,Linear-SP,RL,Collusion},
    xtick=data,
    ymin=0,
    ymajorgrids=true,
    grid style=dashed,
    legend style={
      font=\footnotesize,
      legend columns=2,     
      transpose legend=true, 
      at={(0.5,-0.25)},     
      anchor=north,
      /tikz/every even column/.append style={column sep=2pt}
    },
    title={Insurance Premia Convergence},
]

\addplot[fill=myGreenL] coordinates {
  (Free Market,5.42) (Linear-SP,4.86) (RL,4.86) (Collusion,12.39)
};
\addplot[fill=myGreenM] coordinates {
  (Free Market,5.30) (Linear-SP,4.81) (RL,4.81) (Collusion,8.60)
};

\addplot[fill=myRedL] coordinates {
  (Free Market,5.46) (Linear-SP,5.47) (RL,5.47) (Collusion,8.60)
};
\addplot[fill=myRedM] coordinates {
  (Free Market,5.35) (Linear-SP,5.36) (RL,5.36) (Collusion,12.39)
};

\addplot[fill=myBlueL] coordinates {
  (Free Market,10.11) (Linear-SP,10.98) (RL,10.98) (Collusion,20)
};
\addplot[fill=myBlueM] coordinates {
  (Free Market,10.02) (Linear-SP,10.82) (RL,10.82) (Collusion,20)
};

\legend{
  $P^{A}_{low}$,
  $P^{B}_{low}$,
  $P^{A}_{middle}$,
  $P^{B}_{middle}$,
  $P^{A}_{high}$,
  $P^{B}_{high}$
}

\end{axis}
\end{tikzpicture}
We note that the RL policy leads firms to converge to prices similar to those resulting from the linear policy. This is likely due to these fairness brackets being optimal, as uncovered in Section~\ref{sec:ablation}. However, the SP is able to improve welfare further by lowering taxes in those brackets, allowing firms to keep more of their profit for achieving optimal fairness.

\section{A Probabilistic Bound on Global Fairness}
\label{appendix:proxy}
With more assumptions, we derive a tighter, albeit probabilistic bound on opt-out gaps as a result of bounds on local gaps.
\begin{proposition}[Local fairness \(\Rightarrow\) global fairness]
We wish to bound, with high probability, the deviation
  $$\bigl|\,\Pr(F=0 \mid A =i) - \Pr(F=0 \mid A=k)\bigr|$$
  $$\;=\;
  \left| \sum_{j=1}^{N} \Bigl(\Pr(F=j \mid A=k) - \Pr(F=j \mid A=i)\Bigr) \right|$$
  $$\;=\;
  \Bigl|\sum_{j=1}^{N} X_j\Bigr|,$$
where we set the random variables
\[
\begin{aligned}
  X_j \;:=\; \Pr(F=j \mid A=k)\;-\;\Pr(F=j \mid A=i),\\
  \qquad j = 1,\dots, N.
\end{aligned}
\]
\medskip
\noindent
\textbf{Assumptions.}
\begin{enumerate}
  \item The variables $\{X_j\}_{j=1}^{N}$ are independent.
  \item Each $X_j$ is bounded: $\lvert X_j\rvert \le \epsilon$.
  \item Optionally, $\operatorname{Var}[X_j] \le \sigma^2$ (needed only for Bernstein).
\end{enumerate}

\paragraph{a)\;Hoeffding’s Inequality (bounded differences)}
Because the $X_j$’s are bounded in $[-\epsilon,\epsilon]$, Hoeffding’s inequality gives, for a candidate bound \(t\),
\[
  \Pr\!\Bigl(\bigl|\textstyle\sum_{j=1}^{N} X_j\bigr| \,\ge t\Bigr)
  \;\le\;
  2 \exp\!\Bigl(-\tfrac{2 t^{2}}{N\,(2\epsilon)^{2}}\Bigr).
\]
Setting
\(
  t = 2\epsilon\,\sqrt{\tfrac{N}{2}\,\log\!\bigl(\frac{2}{\delta}\bigr)}
\)
yields, with probability at least $1-\delta$,
\begin{equation}\label{eq:hoeffding-final}
\begin{aligned}
  \Bigl|\sum_{j=1}^{N} \bigl(P(F=j \mid A=k) - P(F=j \mid A=i)\bigr)\Bigr|\\
  \;\le\;
  2\epsilon\,\sqrt{\frac{N}{2}\,
    \log\!\Bigl(\tfrac{2}{\delta}\Bigr)}.
\end{aligned}
\end{equation}

\paragraph{b)\;Bernstein’s Inequality (variance information available)}
If in addition $\operatorname{Var}[X_j] \le \sigma^{2}$, Bernstein’s inequality states
\[
  \Pr\!\Bigl(\bigl|\textstyle\sum_{j=1}^{N} X_j\bigr| \,\ge t\Bigr)
  \;\le\;
  2\exp\!\Bigl(
      -\frac{t^{2}}
            {2N\sigma^{2} + \tfrac{2}{3}\epsilon\,t}
  \Bigr),
\]
which can be tighter when most differences are very small (i.e.\ $\sigma^{2} \ll \epsilon^{2}$).

\paragraph{High–Probability Bound}
From~\eqref{eq:hoeffding-final}, we conclude that, with probability at least $1-\delta$,
\[
\begin{aligned}
  \bigl|\,\Pr(F=0 \mid A=i) - \Pr(F=0 \mid A= k)\bigr|
  \;\le\;\\
  2\epsilon\,\sqrt{\frac{N}{2}\,
    \log\!\Bigl(\tfrac{2}{\delta}\Bigr)}.
\end{aligned}
\]
\end{proposition}
\noindent This Hoeffding bound, while not a convergence guarantee, is sufficient to demonstrate the tight relationship between local and global fairness.

\section{SAC Hyperparameters}
\label{appendix:sac_hypers}
We outline the hyperparameters used with SAC to train our Social Planner agent. They are the same as the default hyperparameters encouraged by~\citet{softactorcritic}.
\rowcolors{2}{white}{gray!05}
\begin{tabular}{l c}
\multicolumn{2}{c}{\cellcolor{gray!20}\textbf{\texttt{SAC} Hyperparameters}}\\
\toprule
& \textbf{Hyperparameters}\\
\midrule
Hidden-layer sizes (actor/critic) & $[256,256]$\\
Activation function & \texttt{ReLU}\\
Learning rate $\eta$ & $0.0003$\\
Batch size & $256$\\
Replay-buffer size & $1M$\\
Discount factor $\gamma$ & $0.00$ (RL)\\
Target-update coef.\ $\tau_{\text{soft}}$ & $0.005$\\
Entropy coef.\ $\alpha_{\text{ent}}$ & $auto$\\
Updates per env step & $1$\\
Warm-up steps & $100$\\
Total training steps & $2M$\\
Random seed & $[0,1,2,3,4]$\\
\bottomrule
\end{tabular}
\label{tab:sac_hypers}

\end{document}